\documentclass[a4paper,conference]{IEEEtran}
\ifCLASSINFOpdf
\else
\fi
\usepackage{helvet} 
\usepackage{courier}  
\usepackage[square,sort,comma,numbers]{natbib}

\usepackage[utf8]{inputenc} 
\usepackage[T1]{fontenc}    
\usepackage{hyperref}

\usepackage{booktabs}       
\usepackage{amsfonts}       
\usepackage{nicefrac}       
\usepackage{microtype}      

\usepackage{epsfig}
\usepackage{amsmath}
\usepackage{amssymb}
\usepackage{adjustbox}

\usepackage{amsthm}  
\usepackage{wasysym}
\usepackage[ruled,vlined,algo2e]{algorithm2e}
\providecommand{\SetAlgoLined}{\SetLine}
\usepackage{color}
\usepackage{dsfont}	
\usepackage{wrapfig}
\usepackage{footnote}
\usepackage{epstopdf}
\usepackage{enumitem}
\usepackage{subfigure}
\usepackage{caption}
\usepackage{comment}
\usepackage{multirow}
\usepackage{algorithm}
\usepackage{algorithmic}

\def\ie{\emph{i.e., }}

\def\vs{\emph{vs. }}
\def\wrt{\emph{w.r.t. }}
\def\etal{\emph{et al. }}

\def\fw{Frank-Wolfe }

\usepackage{pifont}

\DeclareMathOperator*{\argmin}{arg\,min}

\usepackage{mathtools}

\makeatletter
\newcommand*{\rom}[1]{\expandafter\@slowromancap\romannumeral #1@}
\makeatother
\makeatletter
\newcommand\footnoteref[1]{\protected@xdef\@thefnmark{\ref{#1}}\@footnotemark}
\makeatother
\newcommand{\bfsection}[1]{\vspace*{0.1cm}\noindent\textbf{#1.}}

\newcommand{\sgn}{\text{sgn}}

\newtheorem{defi}{Definition}
\newtheorem{thm}{Theorem}

\usepackage{bm}
\usepackage[utf8]{inputenc}
\DeclareUnicodeCharacter{2212}{-}
\usepackage{caption}
\usepackage{float}
\usepackage{subfiles} 

\usepackage{stackengine}
\makeatletter
\newcommand\fs@betterruled{%
  \def\@fs@cfont{\bfseries}\let\@fs@capt\floatc@ruled
  \def\@fs@pre{\vspace*{8pt}\hrule height.8pt depth0pt \kern2pt}%
  \def\@fs@post{\kern2pt\hrule\relax}%
  \def\@fs@mid{\kern2pt\hrule\kern2pt}%
  \let\@fs@iftopcapt\iftrue}
\floatstyle{betterruled}
\restylefloat{algorithm}
\makeatother

\begin{document}
%
\title{RNN Training along Locally Optimal Trajectories via \fw Algorithm}


\author{\IEEEauthorblockN{Yun Yue\IEEEauthorrefmark{1},
Ming Li\IEEEauthorrefmark{1},
Venkatesh Saligrama\IEEEauthorrefmark{2}
and
Ziming Zhang\IEEEauthorrefmark{1}}
\IEEEauthorblockA{\IEEEauthorrefmark{1}
Worcester Polytechnic Institute, 
Worcester, MA 01609
}
\IEEEauthorblockA{\IEEEauthorrefmark{2}
Boston University, Boston, MA 02215
}
Email: \{yyue,  mli12, zzhang15\}@wpi.edu, srv@bu.edu}

\maketitle

\begin{abstract}

We propose a novel and efficient training method for RNNs by iteratively seeking a local minima on the loss surface within a small region, and leverage this directional vector for the update, in an outer-loop. We propose to utilize the Frank-Wolfe (FW) algorithm in this context. Although, FW implicitly involves normalized gradients, which can lead to a slow convergence rate, we develop a novel RNN training method that, surprisingly, even with the additional cost, the overall training cost is empirically observed to be lower than back-propagation. Our method leads to a new \fw method, that is in essence an SGD algorithm with a restart scheme. We prove that under certain conditions our algorithm has a sublinear convergence rate of $O(1/\epsilon)$ for $\epsilon$ error. We then conduct empirical experiments on several benchmark datasets including those that exhibit long-term dependencies, and show significant performance improvement. We also experiment with deep RNN architectures and show efficient training performance. Finally, we demonstrate that our training method is robust to noisy data.
\end{abstract}


%
\IEEEpeerreviewmaketitle

\section{Introduction}\label{sec:intr}
Consider the problem of training RNNs based on minimizing empirical risk over minibatches, ${\cal B}$ that are sampled uniformly at random from training examples ${\cal X} \times {\cal Y}$ of feature-label pairs $(x,y)$ over $M$ time steps. Let us denote the instance at m-step as $x_m$, and the hidden state as $\mathbf{z}_m$. The batch-averaged empirical risk can be written as: 
\begin{align}\label{eqn:obj}
    & \min_{\omega} \left [F(\omega) \stackrel{def}{=} \mathbb{E}_{\mathcal{B}\sim\mathcal{X}\times\mathcal{Y}} f(\mathcal{B}; \omega) \stackrel{def}{=} \sum_{(x,y)\sim\mathcal{B}}\ell(y, \mathbf{z}_M; \omega_{\ell}) \right ] \nonumber \\
    & \mbox{s.t.} \,\,\,\mathbf{z}_m = h(x_m, \mathbf{z}_{m-1}; \omega_h), \,\,\,\forall\,\,\, m\in[M],
\end{align}
where $\omega=\{\omega_{\ell}, \omega_h\}$ denotes the RNN weights, 
$\ell, h$ denotes the loss and (nonconvex) state transmission functions parameterized by $\omega_{\ell}, \omega_h$, and $\mathbb{E}$ denotes the expectation operator.

\bfsection{Vanishing and Exploding Gradients}
Training stability of RNNs is regarded as a fundamental aspect, attracting much attention in the literature \cite{cohen1983absolute, guez1988stability, hopfield1982neural, hopfield1984neurons, kelly1990stability, matsuoka1992stability, hui1992dynamical, michel1990analysis, collins2016capacity, miller2018stable}. 
In this context, gradient explosion/decay is identified as one of the key reasons that prevent RNNs from being trained efficiently and effectively, where the gradient magnitude is either too small or too large, leading to severe training instability \cite{pascanu2013difficulty}. 
This issue has been attributed to:
\begin{itemize}[nosep, leftmargin=7mm]
    \item[{\em P1.}] The number of time steps, $M$, is large where long-term dependencies exist among the data;
    \item[{\em P2.}] The state transmission function, $h$, involves multiple hidden states such as in deep RNNs;
    \item[{\em P3.}] The data samples, $\mathcal{X}$, are very noisy or the true signal is weak.
\end{itemize}
\begin{wrapfigure}{r}{.4\linewidth}
	\vspace{-5mm}
	\begin{center}
		\includegraphics[width=\linewidth]{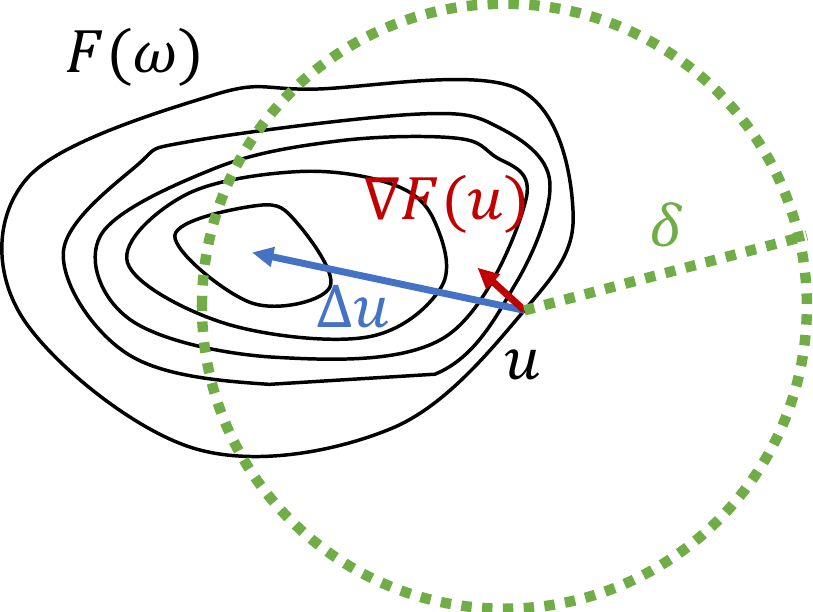}
		\vspace{-5mm}
		\caption{\footnotesize{Proposed method.}}
		\label{fig:idea}
	\end{center}
	\vspace{-15pt}
\end{wrapfigure}
\noindent
\bfsection{Proposed Method}
Different from prior works that propose methods based on novel designs \cite{chang2018antisymmetricrnn} or architectures \cite{hochreiter1997long} as a means to mitigate gradient decay or explosion, we propose to directly modify the well-known back-propagation algorithm. 

At a high-level, we propose to estimate the {\em stable} (approximate) gradients in RNNs. In Fig. \ref{fig:idea}, $u$ denotes the current realization for function $F(\omega)$ whose gradient is $\nabla F(u)$. $\Delta u$ denotes the desired output vector that points towards the local minimum from $u$, and $\delta\geq0$ denotes the radius of the search region in the parameter space centered at $u$ (denoted by the dotted circle). Obviously, $\nabla F(u)$ and $\Delta u$ could be quite different, and our goal is to learn $\Delta u$, by looking around in a sufficiently small neighborhood. 

\bfsection{Trust-Region \vs Projected Gradient \vs Frank-Wolfe}
All of the three methods are potentially applicable for our learning purpose. Trust-region optimization~\cite{byrd1987trust,fortin2004trust,alexandrov1998trust} usually utilizes quadratic approximation to locate a local minimum given the current solution. In deep learning, however, computing the Hessian matrix in the high dimensional space is very challenging. If simplified with the identity matrix, the corresponding closed-form solution is equivalent to the $\ell_2$-normalized gradient. Projected gradient \cite{boyd2004convex} for training RNNs may lead to slow convergence due to the vanishing/exploding gradients, because the inner loop for locating the local minimum relies on the RNN gradients as well. In contrast, we propose to utilize the Frank-Wolfe (FW) algorithm \cite{frank1956algorithm}, one of the simplest and earliest iterative algorithms for constrained optimization due to its projection-free property and large-scale applicability. 
FW maintains a feasible solution satisfying the constraints, without the need to explicitly incorporate the constraints. Nevertheless, we are generally agnostic to the approach used, and our focus is on RNN training methods that are based on locally optimized trajectory. We collect a few salient aspects of our method and FW in particular to further build intuition.

\bfsection{Normalized Gradient \vs Frank-Wolfe}
Since the magnitude of vanishing or exploding gradients limits RNN trainability, one plausible approach is to normalize gradients. While normalization has been been explored in practice for training deep models \cite{Zhang_2018_CVPR}, results in \cite{Zhang2020Why} demonstrate that the iteration complexity of $\ell_2$-normalized stochastic gradient descent (SGD) is $O(1/\epsilon^4)$ in order to achieve $\epsilon$-stationary solutions. In contrast, although normalization is inherent in the \fw algorithm as well \wrt $\ell_p (p\geq1)$-norm constraints, such normalization guarantees the solution per iteration to be inside the local region towards a local minimum. Additionally, we show that our proposed \fw algorithm converges faster. 

\bfsection{Stable Direction} 
Gradient is singularly local in information, and as such, it does not fully capture the landscape, say, within a small ball around the point. Obviously, if the gradients were small within the entire ball, this would correspond to essentially being on a flat surface. In other cases, we are often in a situation, where the gradient vectors change rapidly, and one can make progress ``if it was possible to look around in suitable neighborhood.'' A measure for changing landscape within a closed convex set, $\mathcal{C}$, for convex smooth functions is the curvature constant $M$ \cite{jaggi2010simple}. For a smooth convex function, $g(\cdot)$, its curvature is defined as the maximum Bregman Divergence, namely,
\begin{align}\label{eqn:M}
    M=\max_{\begin{array}{c} x,y,s \in \mathcal{C}\\ y=(1-\gamma)x + \gamma s \end{array}} \frac{2}{\gamma^2} (g(y)-g(x) - \langle \nabla g(x), y-x\rangle ).
\end{align}
Note that the curvature $M$ captures the notion of maximum error over all linear approximation errors within the set $\mathcal{C}$, and as such it can be significantly different from the gradient at any specific location. The notion of curvature is intimately tied to the Frank-Wolfe algorithm. A remarkable fact is that the Frank-Wolfe algorithm, is {\it affine invariant}, and with step size $\gamma_k=2/(k+1), \forall k\geq1$, asserts that, at iteration $k$, we can guarantee $g(x^k)-g(x^*) \leq \frac{M}{k+2}$, where $g(x^*)$ is the minimum in the set $\mathcal{C}$. Now, in our context of RNNs, the function $F(\omega)$ is not globally convex, and this convergence rate is no longer true globally. Nevertheless, if one were in a position that satisfies local convexity within a sufficiently small ball of radius $\delta$, it is then possible to assert this fact. While, even this cannot be guaranteed, the situation only demands that along the path trajectory of the updates, we are in a locally convex region. Additionally, one can often realize local convexity by adding a suitably small quadratic regularizer, which is adapted to the current location in the parameter space. In Sec.~\ref{sec:analysis}, under suitable technical conditions, which fall short of global convexity, we show that our proposed algorithm converges at a sublinear rate of $O(1/\epsilon)$ to achieve $\epsilon$ approximation error. 

The key insight from Eq. \ref{eqn:M} for training RNNs is that {\em for convergence guarantees, local curvature constant, rather than the noisy gradient at the current location is of relevance.} This provides us a means to mitigate vanishing gradients, through leveraging \fw algorithm.

\bfsection{Contributions}
Our key contributions in this paper are:
\begin{itemize}[leftmargin=*]
    \item We propose a novel yet simple RNN optimizer based on the \fw method;
    
    \item We theoretically analyze the convergence of our algorithm and its benefits in RNN training;
    
    \item We empirically conduct comprehensive experiments to demonstrate the effectiveness and efficiency of our algorithm in various settings that cover all the scenarios of P1, P2, P3.
    
\end{itemize}

\section{Related Work}
Below we only summarize the related works in the literature of RNNs.

    \bfsection{Optimization in RNNs} 
    Truncated backpropagation through time (TBPTT) \cite{jaeger2002tutorial} is a widely used technique to avoid vanishing/exploding gradients in RNNs by controlling the maximum number of time steps in gradient calculation, although it has been demonstrated to be not robust to long-term dependencies \cite{tallec2017unbiasing}. Gradient clipping \cite{pascanu2013difficulty, li2018independently} is another common technique to prevent the exploding (not vanishing) gradients\footnote{Therefore, in practice vanishing gradients are more often for performance degradation of RNNs.} by, for instance, rescaling gradients when their norms are over a predefined threshold. It has been proved in \cite{Zhang2020Why} that gradient clipping can accelerate the training of deep neural networks. The scheme of initializing recurrent weight matrices to be identity or orthogonal has been widely studied such as \cite{le2015simple, arjovsky2016unitary, jing2017tunable, jose2017kronecker, mhammedi2017efficient, wisdom2016full, vorontsov2017orthogonality}. Weight matrix reparametrization has been explored as well \cite{zhang2018stabilizing}. Some other optimization approaches for training RNNs are proposed as well such as real-time recurrent learning (RTRL) \cite{williams1989learning}. To the best of our knowledge, trust-region (although it has been explored in other applications of deep learning such as reinforcement learning \cite{schulman2015trust}) or projected gradient methods have not been studied widely as an RNN optimizer, which we consider as one of our future works.
    
    \bfsection{Novel Network Architecture Development for RNNs} 
    Recently there are significant amount of works on developing variants of RNNs such as, just to name a few, long short-term memory (LSTM) \cite{hochreiter1997long}, gated recurrent unit (GRU) \cite{cho2014properties, kanai2017preventing}, Fourier recurrent unit \cite{zhang2018learning}, UGRNN \cite{2016arXivUGRNN}, FastGRNN \cite{kusupati2018nips}, unitary RNNs \cite{arjovsky2016unitary, jing2017tunable, 2018SpectralRNN, 2017oRNN, pennington17}, deep RNNs \cite{pascanu2013construct, zilly2017recurrent, mujika2017fast}, linear RNNs \cite{quasirnn, lei2018sru, balduzzi2016strongly}, residual/skip RNNs \cite{Jaeger07,Bengio2013AdvancesIO,chang2017dilated,campos2017skip,kusupati2018nips}, ordinary differential equation (ODE) based RNNs \cite{talathi2015improving, niu2019recurrent, chang2018antisymmetricrnn, kusupati2018nips, chen2018neural, Yulia2019OdeRNN, kag2019rnns}. For instance, FastGRNN \cite{kusupati2018nips} feed-forwards state vectors to induce skip or residual connections, to serve as a middle ground between feed-forward and recurrent models, and to mitigate gradient decay. Incremental RNNs (iRNNs) \cite{kag2019rnns} are developed based on ODE with theoretical guarantee of identity gradients in the intermediate steps in chain rule for gradient calculation. Independently RNN (IndRNN) \cite{li2018independently} is a network structure where neurons in the same layer are independent of each other and multiple IndRNNs can be stacked to construct a deep network.

\section{\fw RNN Optimizer}\label{sec:method}

\begin{algorithm}[t]
    \SetAlgoLined
    \SetKwInOut{Input}{Input}\SetKwInOut{Output}{Output}
    \Input{objective $f$, norm $p$, local radius $\delta_t, \forall t$, max numbers of iterations $K, T$}
    \Output{RNN weights $\omega$}
    \BlankLine
    Randomly initialize $\omega_0$;
    
    \For{$t=1,\cdots,T$}{
        $\Delta\omega_{t,0}\leftarrow\mathbf{0}$; 
        
        \For{$k=1,\cdots,K$}{
            
            $s_{t,k}\leftarrow\argmin_{s\in\mathcal{C}(p,\delta_t)}\langle s, \nabla_{\Delta\omega} F(\omega_{t-1}+\Delta\omega_{t,k-1})\rangle$;
            
            $\Delta\omega_{t,k} \leftarrow (1-\frac{1}{k})\Delta\omega_{t,k-1} + \frac{1}{k}s_{t,k}$;

            
        }
        
    $\omega_t \leftarrow \omega_{t-1} + \eta\Delta\omega_{t,K}$; 
    }
    \Return $\omega_T$;
    \caption{Frank-Wolfe RNN Optimizer}\label{alg:fw-rnn}
\end{algorithm}

Recall that in our proposed method, given the current point $\omega$ in the parameter space, we attempt to solve a local minima in small ball around $\omega$. To this end, we consider the following general constrained optimization problem, which is also a  central aspect of the Frank-Wolfe method:
\begin{align}\label{eqn:inner-loop}
    \min_{\|\Delta\omega\|_p \leq \delta} F(\omega+\Delta\omega), 
\end{align}
where $p>1$ is any $\ell_p$ norm. 
For our proposed method (see Alg. \ref{alg:fw-rnn}), this constrained optimization is carried out in an inner loop, by means of the FW method. Once, a satisfactory point is reached (which usually is quite fast), we then update the parameters in an outer loop. 

\bfsection{\fw Algorithm}
We apply the stochastic \fw algorithm to solve Eq. \ref{eqn:inner-loop}, and list our novel and simple \fw based RNN optimizer in Alg. \ref{alg:fw-rnn}. In contrast to GD, the iterations in the \fw algorithm are as follows, in general:
\begin{equation}\label{eqn:FW}
\begin{gathered}
        s^{(k-1)} =\argmin_{s\in\mathcal{C}}\left\langle s, \nabla f(x^{(k-1)})\right\rangle,\\ x^{(k)} = (1-\gamma_k)x^{(k-1)} + \gamma_k s^{(k-1)}, 
\end{gathered}
\end{equation}
where $\langle\cdot,\cdot\rangle$ denotes the inner product of two vectors, and $\gamma_k\in[0,1]$. For $\ell_p (p\geq1)$-norm constraints, \ie $\mathcal{C}(p, \delta)=\{s \mid \|s\|_p\leq\delta\}$, there exist close-form solutions for $\argmin$ in Eq. \ref{eqn:FW}, as
\begin{equation}
\begin{gathered}
    s^{(k-1)} = -\alpha\cdot \sgn(\nabla f(x^{(k-1)})) \cdot |\nabla f(x^{(k-1)})|^{\frac{p}{q}}, \\ \mbox{s.t.} \; 1/p + 1/q = 1,
\end{gathered}
\end{equation}
where $\sgn$ and $|\cdot|$ denote entry-wise sign and absolute operators, respectively, and $\alpha\geq0$ is a scalar satisfying $\|s^{(k-1)}\|_q = \delta$. In particular,
\begin{itemize}[leftmargin=*]
    \item $p=2 \Rightarrow s^{(k-1)} = -\delta \cdot \frac{\nabla f(x^{(k-1)})}{\|\nabla f(x^{(k-1)})\|_2}$, \ie $\delta$-scaled $\ell_2$ normalized gradient;
    
    \item $p\rightarrow\infty \Rightarrow s^{(k-1)} = -\delta \cdot \sgn(\nabla f(x^{(k-1)}))$, \ie $\delta$-scaled sign gradient.
\end{itemize}

In Alg. \ref{alg:fw-rnn} we set $\gamma_k = \frac{1}{k}$ to ensure that $\Delta\omega$ always lies inside the local region with convergence guarantee \cite{reddi2016stochastic}. Similarly we decrease $\delta_t$ with the learning rate. Note that Adam \cite{kingma2014adam} can be used to update $\omega$. Besides, Alg. \ref{alg:fw-rnn} can be applied to train other deep models as well. Similar two-loop algorithmic structures have been explored in the literature such as \cite{zhang2019time} for training convolutional neural networks (CNNs). 

In this paper we focus on the analysis and experiments with $p=2$, although our analysis generally holds for $p\geq1$. In practice, we utilize $p\rightarrow\infty$ to train IndRNN \cite{li2018independently} and achieve $98.37\%$ on Pixel-MNIST with $K=30$ over 400 epochs, similar to the numbers in Table \ref{tab:indRNN}.

\section{Analysis}\label{sec:analysis}
Zhu \etal in \cite{allen2019convergence} proved that for training RNNs, when the number of neurons is sufficiently large, meaning {\em polynomial} in the training data size and in network depth, then SGD is capable of minimizing the regression loss in the linear convergence rate. In contrast, we discuss the convergence of Alg.~\ref{alg:fw-rnn} for an arbitrary RNN. We start our analysis from a special case with $\ell_2$-normalized gradients as follows:
\begin{thm}[$\ell_2$-Normalized Gradients]
Suppose that the following assumptions hold globally: 
\begin{itemize}[leftmargin=*]
    \item $F$ is lower-bounded, twice differentiable, and $(L_0, L_1)$-smoothness, \ie $\|\nabla^2 F(\omega)\|_2 \leq L_0 + L_1\|\nabla F(\omega)\|_2, \forall \omega, \exists L_1,L_2>0$;
    
    \item There exist a $\tau>0$ such that $\|\nabla f(\mathcal{B}, \omega) - \nabla F(\omega)\|_2\leq\tau, \forall\mathcal{B}, \forall \omega$ holds.
\end{itemize}
Then for $K=1$ in Alg. \ref{alg:fw-rnn}, there exists $\delta>0$ so that in order to achieve $\epsilon$-stationary points the iteration complexity for Alg.~\ref{alg:fw-rnn} is upper bounded by $O(1/\epsilon^4)$, at least.
\end{thm}
To prove this theorem, please refer to Thm. 7 in \cite{Zhang2020Why}. Below we extend our analysis to general cases.

\begin{defi}[Star-convexity \cite{zhou2018sgd}]
Let $\omega^*$ be a global minimizer of a smooth function $F$. Then, $F$ is said to be star-convex at a point $\omega$ provided that $F(\omega) − F(\omega^*) + \langle \omega^* − \omega, \nabla F(\omega)\rangle\leq 0, \forall \omega$.
\end{defi}

\begin{thm}[Convergence]\label{thm:convergence}
Let $\{\omega_t\}_{t\in[T]}$ be the sequence of the weight updates from Alg.~\ref{alg:fw-rnn}. Suppose 
\begin{itemize}[leftmargin=7mm]
    \item[A1.] $F$ in Eq. \ref{eqn:obj} is {\em locally convex} within each radius $\delta_t, \forall t\in[T]$ \wrt $\ell_2$ norm centered at $\omega_{t-1}$;
    
    \item[A2.] $F$ in Eq. \ref{eqn:obj} is also {\em star-convex}, given a global minimizer $\omega^*$, \ie $F$ is lower bounded by $F(\omega^*)$;
    
    \item[A3.] $F$ in Eq. \ref{eqn:obj} is differentiable and its gradient is Lipschitz continuous with constant $L > 0$, \ie $\|\nabla F(\omega_1) - \nabla F(\omega_2)\|_2 \leq L\|\omega_1 - \omega_2\|_2, \forall \omega_1, \omega_2$;
    
    \item[A4.] $\omega_t, \forall t$ is upper bounded \wrt $\omega^*$, \ie $\|\omega_t - \omega^*\|_2\leq \alpha < +\infty, \forall \omega, \exists \alpha$; 
    
    \item[A5.] It holds that $\beta\leq\|\nabla F(\omega_{t-1}) - \Delta\omega_{t,K}\|_2 \leq (1-L\eta)^{\frac{1}{2}}\|\Delta\omega_{t,K}\|_2, \forall t, \exists\beta>0, \exists\eta\leq\frac{1}{L}$. 
    
\end{itemize}
Then we have that the output of Alg. \ref{alg:fw-rnn}, $\omega_T$, satisfies
\begin{align}\label{eqn:sublinear}
    F(\omega_T) - F(\omega^*) \leq \frac{\|\omega_0 - \omega^*\|_2^2 + 2\eta \rho}{2\eta T}, 
\end{align}
where $\rho=\left(\frac{\alpha}{\beta}-\frac{\eta}{2}\right)(1 - L\eta)\sum_t\|\Delta\omega_{t,K}\|_2^2\in\mathbb{R}$, \ie a real number, and $\omega_0$ denotes the initialization of the network weights. In particular, $\omega_T$ will converge to $\omega^*$ {\em asymptotically} if $\lim_{T\rightarrow+\infty}\sum_{t=1}^T\delta_t^2<+\infty$ holds. Further, if $0 \leq \|\omega_0 - \omega^*\|_2^2 + 2\eta\rho <+\infty$ holds, then $\omega_T$ will converge to $\omega^*$ {\em sublinearly}. 
\end{thm}
\begin{proof}
Based on Assmp. A1, A3 and A5, we have
\begin{align}\label{eqn:local-convexity}
    & F(\omega_t) \leq F(\omega_{t-1}) + \langle\nabla F(\omega_{t-1}), \omega_t-\omega_{t-1}\rangle + \frac{L}{2}\|\omega_t-\omega_{t-1}\|_2^2  \nonumber\\
    &= F(\omega_{t-1}) - \eta\langle\nabla F(\omega_{t-1}), \Delta\omega_{t,K}\rangle + \frac{L\eta^2}{2}\|\Delta\omega_{t,K}\|_2^2
    \nonumber\\
    &= F(\omega_{t-1}) + \frac{\eta}{2}\|\nabla F(\omega_{t-1}) - \Delta\omega_{t,K}\|_2^2 \nonumber\\
    &-\frac{\eta}{2}\|\nabla F(\omega_{t-1})\|_2^2 + \frac{L\eta^2-\eta}{2}\|\Delta\omega_{t,K}\|_2^2\leq F(\omega_{t-1}).
\end{align}
Further, based on Assmp. A2, A4 and A5, we have
\begin{align}\label{eqn:star-convexity}
    & F(\omega_t) - F(\omega^*) \nonumber \leq \langle\nabla F(\omega_{t-1}), \omega_{t-1}-\omega^*\rangle - \frac{\eta}{2}\|\nabla F(\omega_{t-1})\|_2^2 \nonumber\\&+ \frac{\eta}{2}\|\nabla F(\omega_{t-1}) - \Delta\omega_{t,K}\|_2^2 + \frac{L\eta^2-\eta}{2}\|\Delta\omega_{t,K}\|_2^2 \nonumber \\
    &= \frac{1}{2\eta}\Big(\|\omega_{t-1} - \omega^*\|_2^2 - \|\omega_t - \omega^*\|_2^2\Big) \nonumber\\&+ \langle\nabla F(\omega_{t-1})-\Delta\omega_{t,K}, \omega_t-\omega^*\rangle + \frac{L\eta^2-\eta}{2}\|\Delta\omega_{t,K}\|_2^2 \nonumber \\
    &\leq \frac{1}{2\eta}\Big(\|\omega_{t-1} - \omega^*\|_2^2 - \|\omega_t - \omega^*\|_2^2\Big) \nonumber\\&+ \left(\frac{\alpha}{\beta}-\frac{\eta}{2}\right)(1 - L\eta)\|\Delta\omega_{t,K}\|_2^2.
\end{align}
Now based on Eq. \ref{eqn:local-convexity} and Eq. \ref{eqn:star-convexity}, we can complete our proof by
\begin{align}
    &F(\omega_T) - F(\omega^*) \leq \frac{1}{T}\sum_t F(\omega_t) - F(\omega^*) \nonumber\\&\leq \frac{\|\omega_0 - \omega^*\|_2^2}{2\eta T} + \frac{1}{T}\left(\frac{\alpha}{\beta}-\frac{\eta}{2}\right)(1 - L\eta)\sum_t\|\Delta\omega_{t,K}\|_2^2. \nonumber
\end{align}
\end{proof}

Equivalently, Thm. \ref{thm:convergence} states that our algorithm needs $O(\frac{1}{\epsilon})$ updates of $\omega$, independent on the inner loops $K$ in Alg. \ref{alg:fw-rnn}, in order to achieve an $\epsilon$-stationary solution. Note that the constant $\rho$ is highly correlated with the number of inner loops, $K$, in the algorithm. Therefore, empirically it is challenging to tell which choice of $K$ will be the best. From our experiments, we found that often small $K$'s can work better than SGD for training RNNs.

\begin{wrapfigure}{r}{.5\linewidth}
	\vspace{-10mm}
	\begin{center}
		\includegraphics[width=\linewidth]{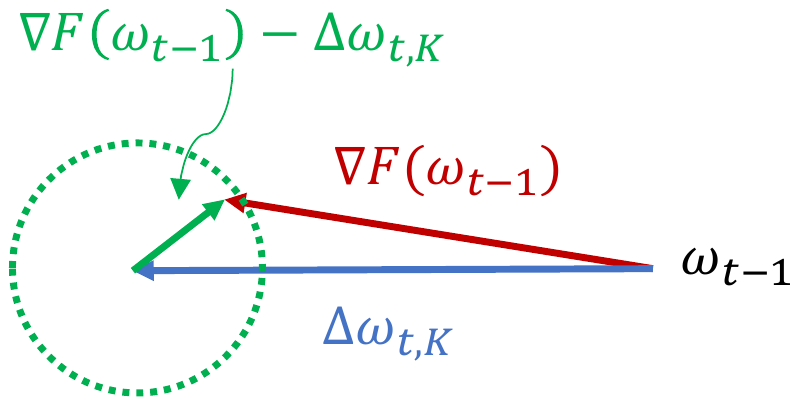}
		\vspace{-5mm}
		\caption{\footnotesize{Geometric interpretation.}}
		\label{fig:A5}
	\end{center}
	\vspace{-10pt}
\end{wrapfigure}
\bfsection{Geometric Interpretation of Assmp. A5}
In Fig. \ref{fig:A5}, we illustrate the geometric relationship between $\nabla F(\omega_{t-1}), \Delta\omega_{t,K}$ and $\nabla F(\omega_{t-1}) - \Delta\omega_{t,K}$ given the current solution $\omega_{t-1}$ and $\beta\rightarrow0$. Clearly, any point within the green dotted circle (with radius $(1-L\eta)^{\frac{1}{2}}\|\Delta\omega_{t,K}\|_2$) will be a candidate for $\nabla F(\omega_{t-1})$ so that A5 holds. Therefore, the angle between any pair of $\nabla F(\omega_{t-1})$ and $\Delta\omega_{t,K}$ should be no more than $\pm 45^o$. In other words, $\nabla F(\omega_{t-1})$ and $\Delta\omega_{t,K}$ should have very similar directions. We verified this on the HAR-2 dataset (see Sec. \ref{sec:exp}) by computing such angles to see the distribution using 200 epochs. Statistically these angles are within $\pm 45^o$ with mean $-4.19^o$ and std $0.51^o$.

\bfsection{Inexact Update in Frank-Wolfe}    
In the stochastic setting with limited $K$ and large-scale data, however, $\nabla F(\omega_{t-1}), \Delta\omega_{t,K}$ are infeasible to compute exactly. To address this problem, we borrow the concept of approximation quality in the inexact \fw method in \cite{jaggi2013revisiting}.
\begin{defi}[\cite{jaggi2013revisiting}]
In the \fw update, $s_{t,k} (\|s_{t,k}\|_2\leq\delta_t)$ is used so that it holds that
\begin{equation}\label{eqn:approximation}
\begin{gathered}
    \langle s_{t,k}, \nabla F(\omega_{t-1}+\Delta\omega_{t,k-1})\rangle \leq \\ \min_{\|s\|_2\leq \delta_t} \langle s, \nabla F(\omega_{t-1}+\Delta\omega_{t,k-1})\rangle + \frac{\lambda M_F}{k+1}, \; \forall k\geq 1,
\end{gathered}
\end{equation}
where $\lambda\geq0$ denotes an arbitrary accuracy parameter that controls the upper bound of the convergence rate linearly from factor $1$ to $(1+\lambda)$, and $M_F$ denotes the curvature constant of $F$.
\end{defi}
Note that Eq. \ref{eqn:approximation} will still hold by setting
    $\lambda = \max_k\left\{\frac{2\delta_t(k+1)}{M_F}\|\nabla F(\omega_{t-1}+\Delta\omega_{t,k-1})\|_2\right\}<+\infty.$
Clearly, as long as $\|\nabla F(\omega_{t-1}+\Delta\omega_{t,k-1})\|_2=O(1/k)$ is met, the convergence of Alg. \ref{alg:fw-rnn} can be always guaranteed. 

\bfsection{Stochastic Frank-Wolfe in Alg. \ref{alg:fw-rnn}}
In our implementation, we utilize the stochastic Frank-Wolfe method \cite{reddi2016stochastic} instead for computational efficiency. Essentially stochastic FW can be considered as a realization of inexact update in FW, and thus all the analysis above holds for this case as well.

\begin{table}[hbt!]
  \caption{Dataset Statistics}
    \label{tab:data}
  \centering
 \begin{tabular}{l r c c c}
      \toprule 
      \textbf{Dataset} & \textbf{\#Train} & \textbf{\#Test}  & \textbf{\#TimeStep}  & \textbf{\#Feature}\\
      \midrule 
      HAR-2 & 7,352 & 2,947 & 128 & 9\\
      Noisy-HAR-2 & 7,352 & 2,947 & 128 & 9\\
      Pixel-MNIST & 60,000 & 10,000 & 784 & 1\\
      Permute-MNIST & 60,000 & 10,000 & 784 & 1\\
      \bottomrule 
    \end{tabular}
\end{table} 

\section{Experiments}\label{sec:exp}
\textbf{Datasets.} 
We test our RNN optimizer on the following benchmark datasets with all the statistics listed in Table \ref{tab:data}:
\begin{itemize}[leftmargin=*]
    \item {\em Adding task:} We strictly follow the adding task\footnote{\url{https://github.com/rand0musername/urnn}} defined by \cite{arjovsky2016unitary, hochreiter1997long} to generate the dataset. There are two sequences with length $T=100$. The first sequence is sampled uniformly at random $\mathcal{U}[0,\,1]$. The second sequence is filled with 0 except for two entries of 1. The two entries of 1 are located uniformly at random position $i_1, i_2$ in the first half and second half of the sequence. The prediction value is the sum of the first sequence between $[i_1, i_2]$. 

    \item {\em Pixel-MNIST \& Permute-MNIST:} Pixel-MNIST refers to pixel-by-pixel sequences of images in MNIST where each 28 $\times$ 28 image is flattened into a 784 time sequence vector, while a random permutation to the Pixel-MNIST is applied to generate a harder time sequence dataset as Permute-MNIST. All datasets are normalized as zero mean and unit variance during training and prediction.
    
    \item {\em HAR-2 \cite{kusupati2018nips}:} HAR-2 was collected from an accelerometer and gyroscope on a Samsung Galaxy S3 smartphone, and all samples are normalized with zero mean and unit variance.
    
    \item {\em Noisy-HAR-2:} This dataset is generated by adding Gaussian noise to HAR-2 with a mean of zero and a variance of two to HAR-2 to evaluate the robustness of our algorithm.
\end{itemize}

\bfsection{Baseline Algorithms and Implementation}
We compare our \fw based RNN optimizer with SGD (gradient clipping involved if necessary) and TBPTT comprehensively, using (1) a vanilla RNN with one-layer transition function consisting of a linear function followed by a $\tanh$ activation (same as the literature. See \cite{kusupati2018fastgrnn}) and (2) IndRNN with six layers\footnote{\url{https://github.com/Sunnydreamrain/IndRNN_pytorch}}. Mean-square-error (MSE) and cross-entropy losses are applied for binary and multi-class classification tasks, respectively. Adam \cite{kingma2014adam} is used as the optimizer for all the methods. We implement our experiments using PyTorch, and run all the experiments on an Nvidia GeForce RTX 2080 Ti GPU with CUDA 10.2 and cuDNN 7.6.5 on a machine with Intel Xeon 2.20 GHz CPU with 48 cores. The code of our optimizer can be found here \footnote{\url{https://github.com/YunYunY/FW_RNN_optimizer}}.

\bfsection{Hyperparameters}
The hyperparameters of the vanilla RNN and IndRNN are the same as the literature \cite{Kag2020RNNs, kusupati2018fastgrnn, li2018independently}, if applicable, as from our experiments they seem to be the best settings. For our optimizer we perform a grid search over several learning rates \{2e-5, 2e-4, 6e-4, 1e-3\}, batch size \{32, 64, 128, 256, 512\} and learning rate decay factor \{0.1, 0.5, 0.9\} on validation data (from a small portion of training data) to ensure that a good parameter is used in our algorithm. We use hidden dimension 128 for MNIST datasets and the Adding task. When training HAR-2 related tasks, we use hidden dimension 80. We report the best test accuracy.





\subsection{Results}
\begin{wrapfigure}{r}{5cm}
	\vspace{-10mm}
	\begin{center}
		\includegraphics[width=\linewidth]{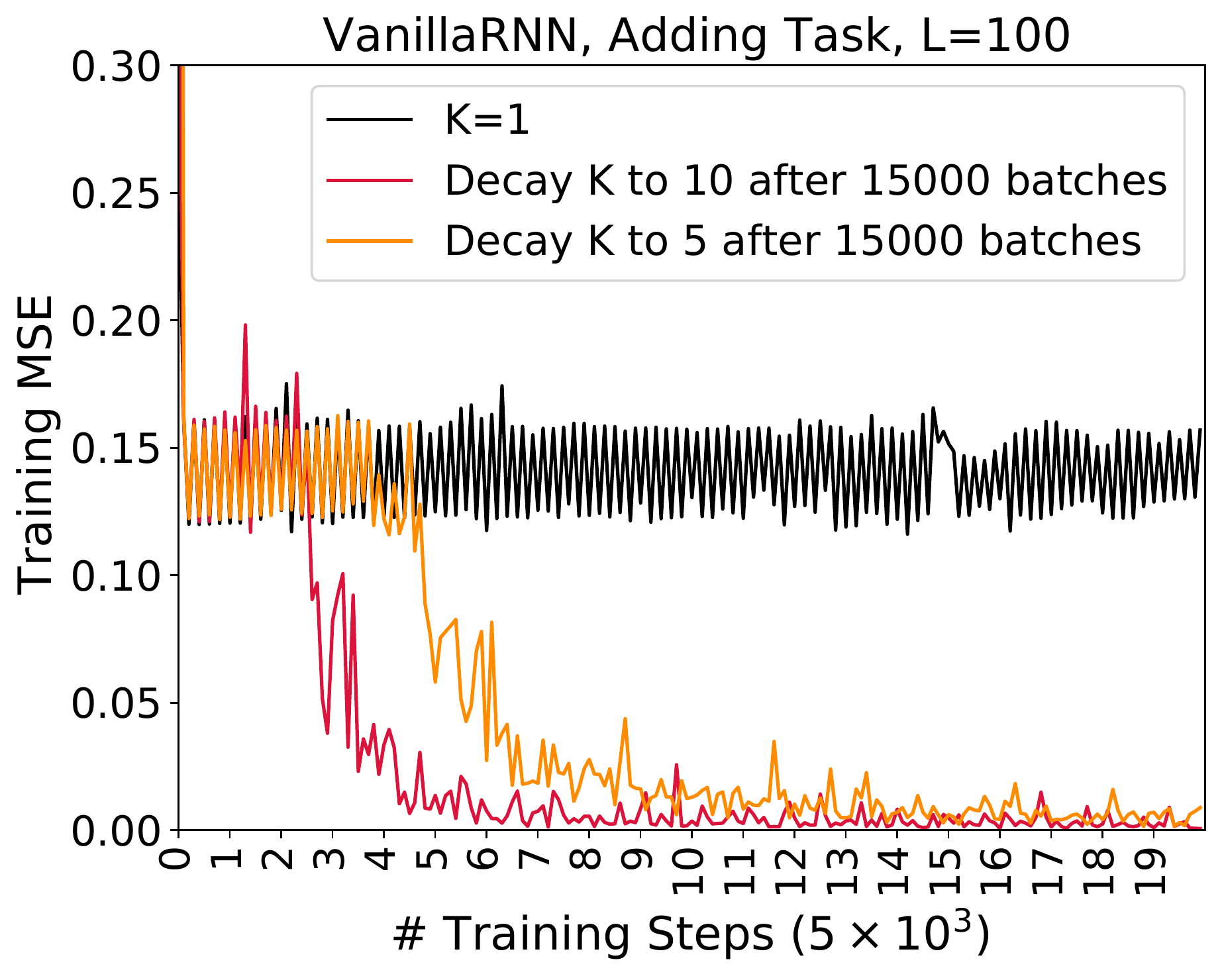}
		\vspace{-5mm}
		\caption{Training loss of Adding Task}
		\label{fig:addtask}
	\end{center}
	\vspace{-3mm}
\end{wrapfigure}

\textbf{Adding Task.}
The Adding task is designed to evaluate the capability of RNNs to capture long-term dependency among the data \cite{2018SpectralRNN, li2018independently, le2015simple, arjovsky2016unitary, arjovsky2016unitary, hochreiter1997long}. Fig. \ref{fig:addtask} illustrates the loss change of our algorithm compared with SGD on this task when the time sequence is long. It is clear that our algorithm converges after a reasonable number of iterations while SGD lost the learning ability in this task. We hypothesize that at the beginning all the algorithms search for a good direction within a certain region. Given sufficient updates later, our algorithm starts to move towards informative directions, leading to significantly fast convergence. 

\bfsection{Vanilla RNN on Pixel-MNIST \& Permute-MNIST}
We further apply our algorithm to Pixel-MNIST and Permute-MNIST and compare it with Adam and TBPTT. The forward and backward steps of TBPTT are set to 196, which means each 784 time sequence is partitioned into 4 segments. For our algorithm, we perform the inner iteration $K$=1 and 5. Since our algorithm can be easily combined with TBPTT, we also include the combination experiments on the two datasets. 

Fig. \ref{fig:p1_pixel} shows the change of training cross-entropy and test accuracy of RNN with the epoch for Pixel-MNIST and Permute-MNIST. Without extra optimization techniques, SGD shows no convergence or very slow convergence. We observe that TBPTT does help the convergence for the Permute-MNIST, however, the performance of TBPTT is only slightly better than the baseline SGD in the Pixel-MNIST case with sporadically increases and decreases of loss. As a contrast, our algorithm shows a faster convergence rate and a much more stable performance on both datasets. When TBPTT is combined with our algorithm, the model achieves faster convergence and higher test accuracy than the baseline for Pixel-MNIST. As for Permute-MNIST, the combination method eventually reaches higher test accuracy with more training epochs. It is worth mentioning that when the inner iteration $K$ increases in our algorithm, the total number of gradient updates needed for convergence decreases. 




\begin{figure}[t]

	\begin{minipage}[b]{0.49\columnwidth}
		\begin{center}
			\centerline{\includegraphics[height=3.5cm, keepaspectratio,]{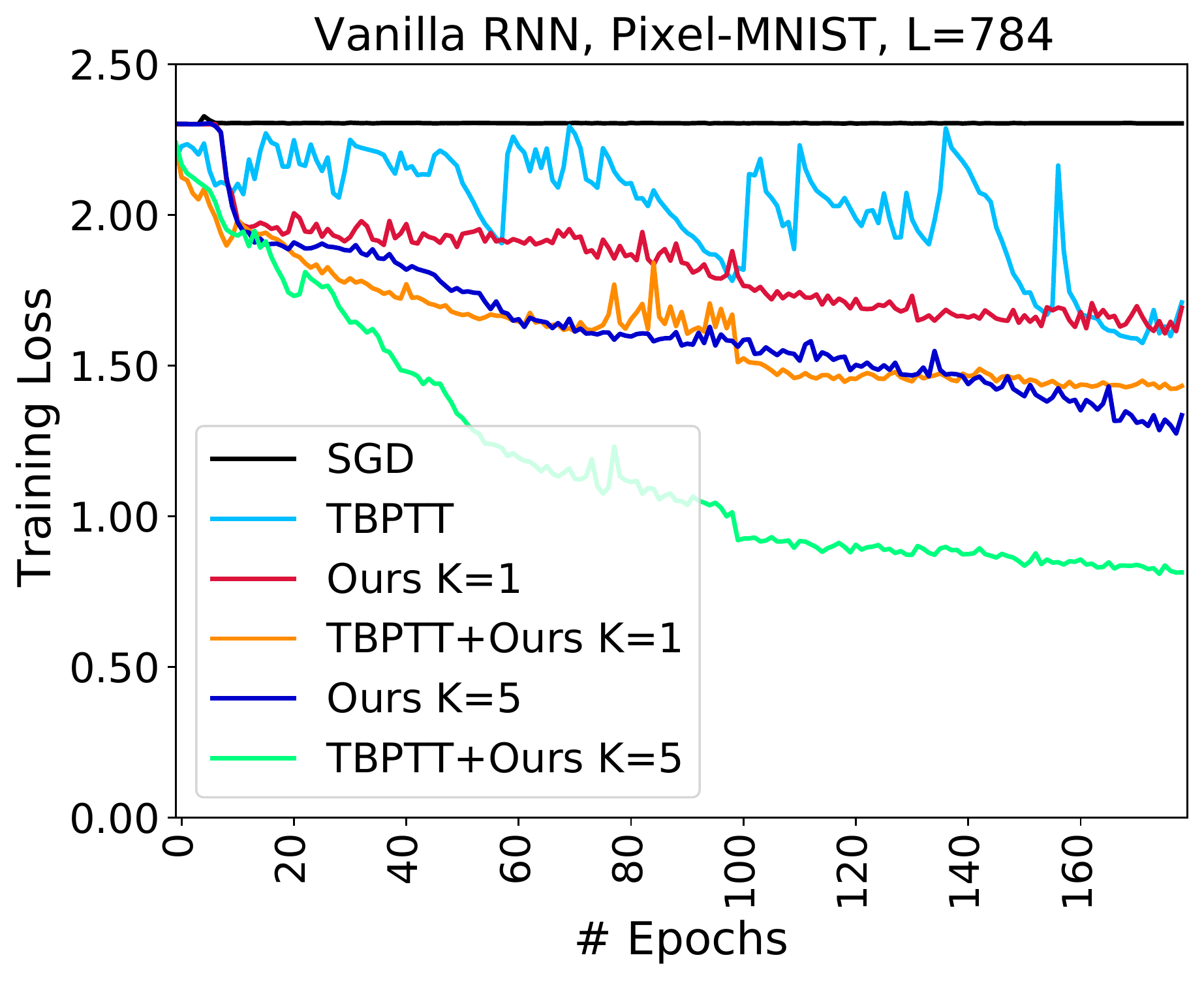}}
		\end{center}
	\end{minipage}
	\begin{minipage}[b]{0.49\columnwidth}
		\begin{center}
			\centerline{\includegraphics[height=3.5cm,keepaspectratio]{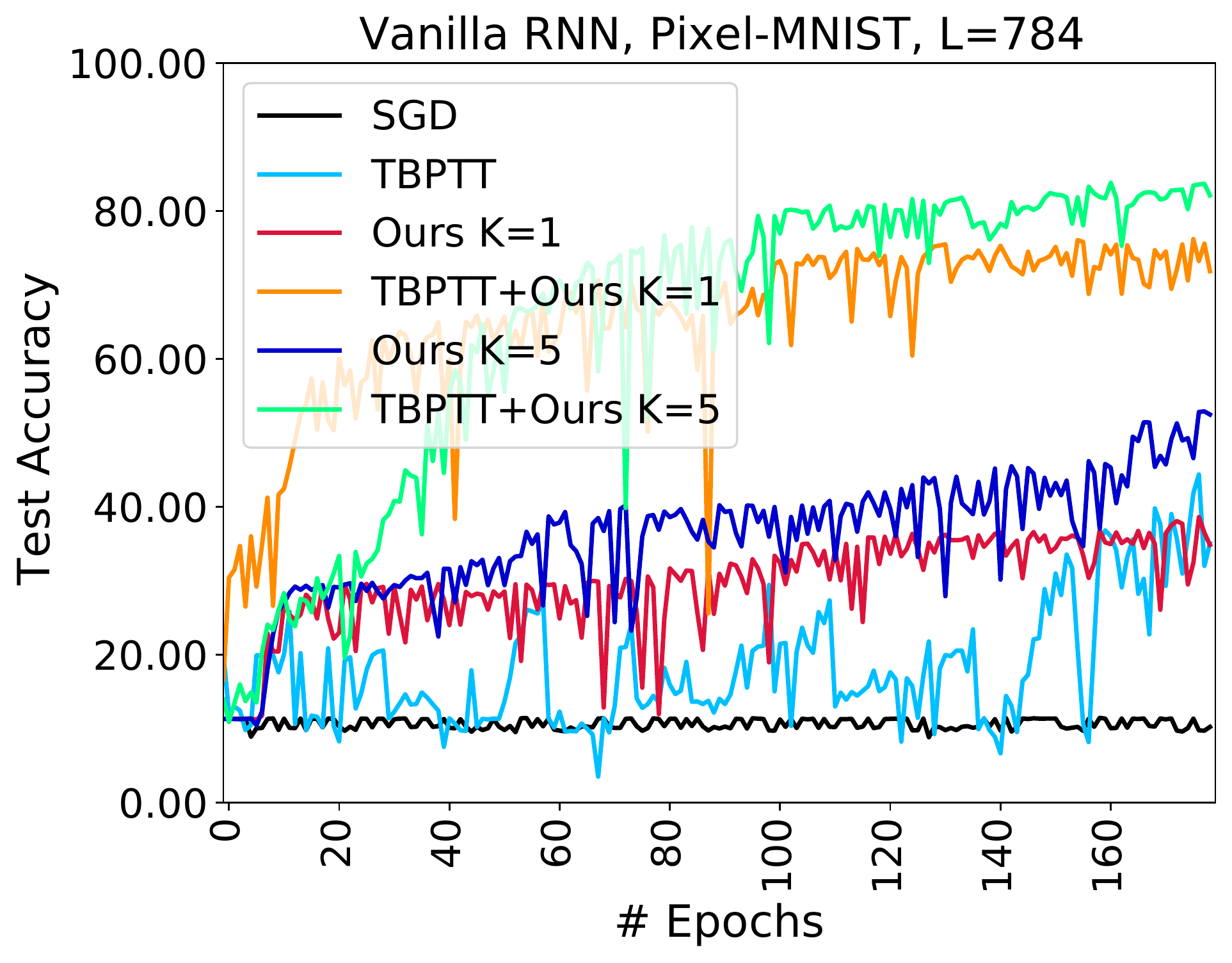}}
		\end{center}
	\end{minipage}
	\hfill
	\begin{minipage}[b]{0.49\columnwidth}
		\begin{center}
			\centerline{\includegraphics[height=3.5cm,keepaspectratio]{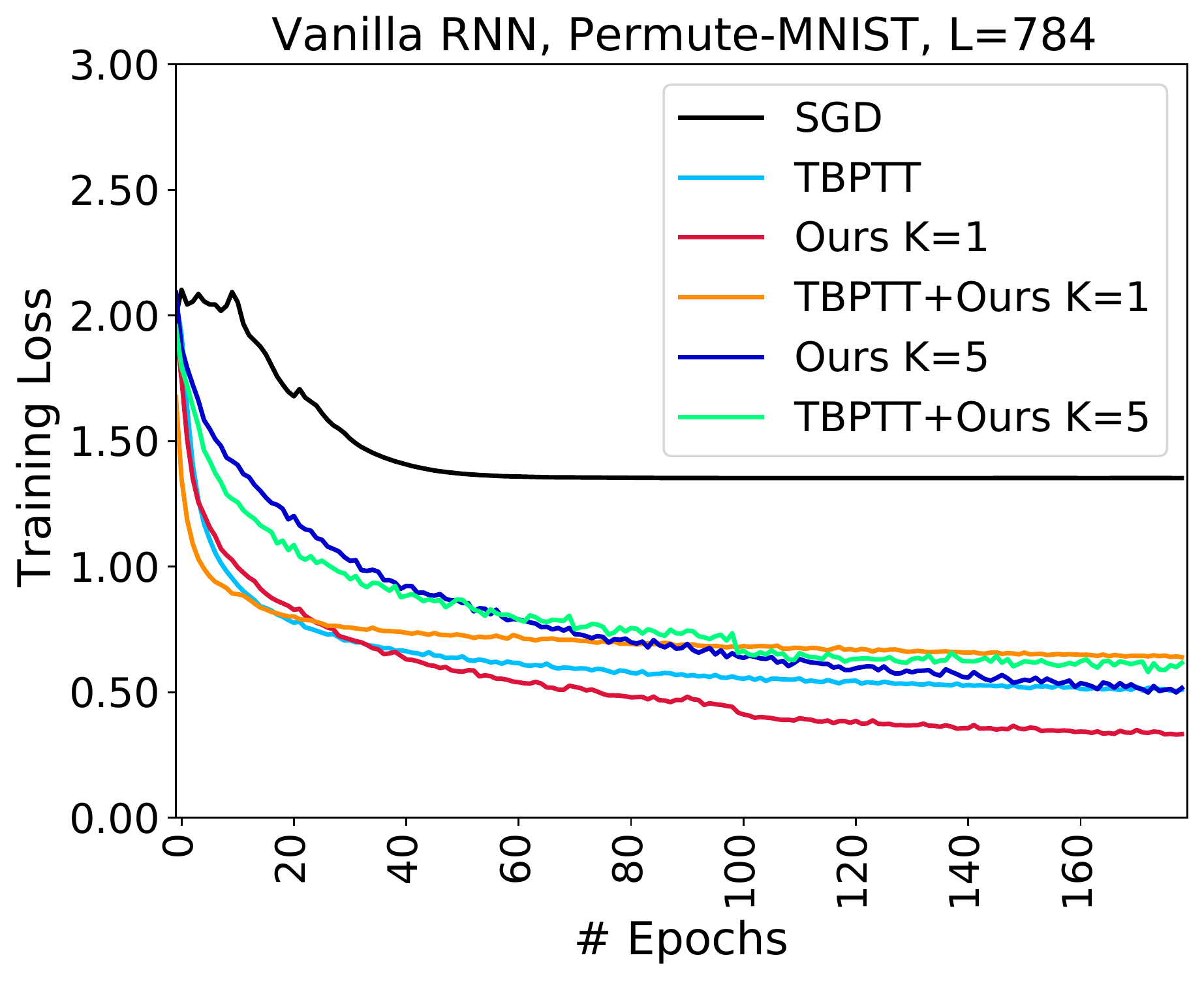}}
		\end{center}
	\end{minipage}
	\begin{minipage}[b]{0.49\columnwidth}
		\begin{center}
			\centerline{\includegraphics[height=3.5cm,keepaspectratio]{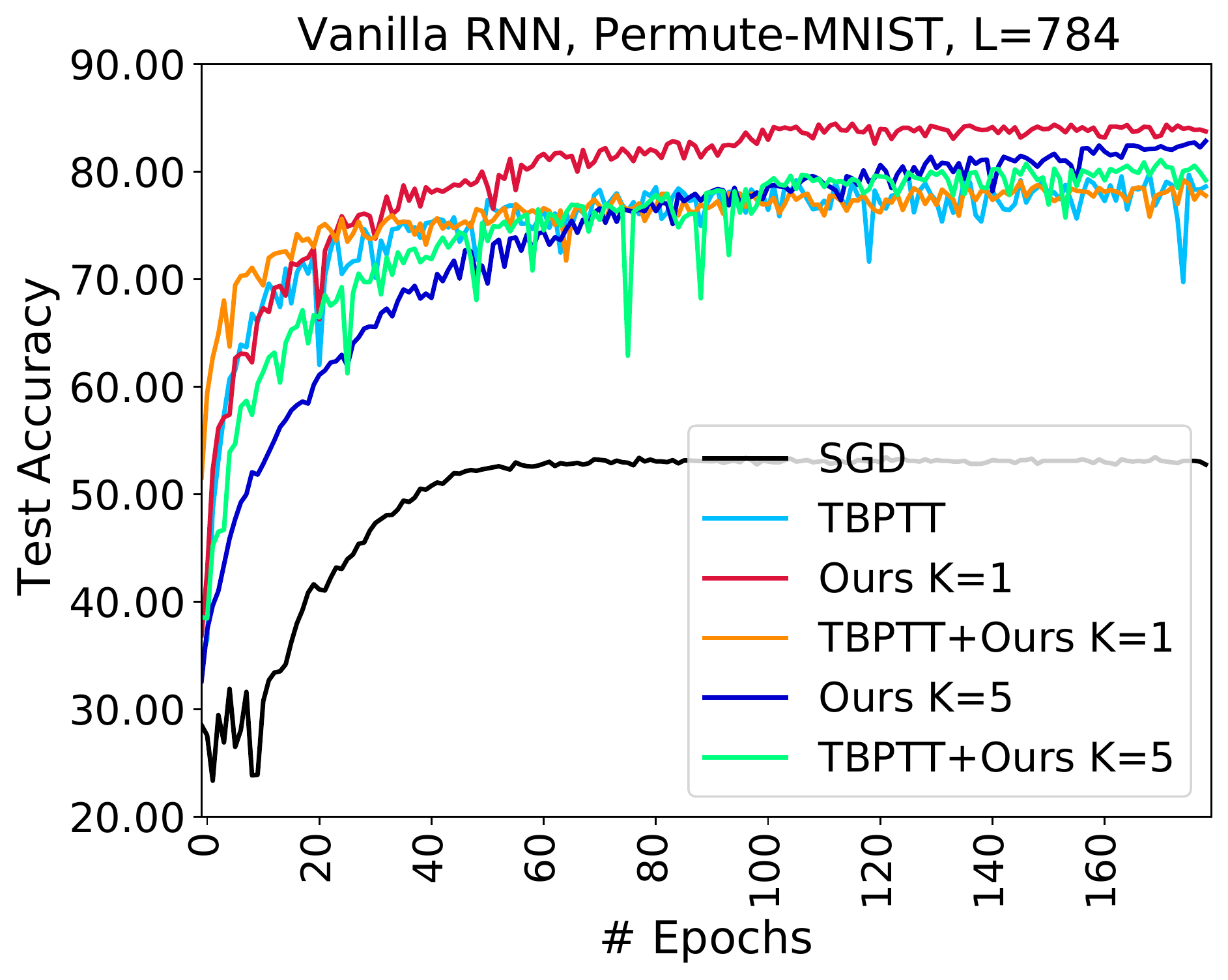}}
		\end{center}
	\end{minipage}
	
	\vspace{-5mm}
    \caption{\footnotesize Training loss and test accuracy on Pixel-MNIST and Permute-MNIST}

\label{fig:p1_pixel}
\end{figure}



\bfsection{IndRNN on Pixel-MNIST \& Permute-MNIST}
This experiment addresses the application of our algorithm on training deep models. Since IndRNN can be stacked to construct a deep network \cite{li2018independently}, we apply a six-layer IndRNN structure with the two benchmark datasets Pixel-MNIST and Permute-MNIST. The model has proved the state-of-the-art performance on the two datasets with Adam optimizer \cite{li2018independently}, thus we also use Adam as the baseline algorithm to compare with ours. The model structure we use follows exactly \cite{li2018independently}.The inner iteration $K$ is tested with values \{1, 5, 10, 30\}.

The training loss and test accuracy results of the two datasets trained with IndRNN are shown in Fig. \ref{fig:indRNN}. IndRNN applied batch normalization (BN) to accelerate Pixel-MNIST training. Due to the truncated training process of TBPTT, the relevant statistics over the mini-batch changes over time. It is not suitable to apply BN to TBPTT. Thus in Fig. \ref{fig:indRNN} we only compare our algorithm with the baseline IndRNN. Our algorithm has the same performance as Adam optimizer. 
In terms of accuracy, our algorithm achieves comparable test accuracy after 400 epochs using about half training time when $K$=5.

\begin{figure}[hbt!]
    \hfill
	\begin{minipage}[b]{0.45\columnwidth}
		\begin{center}
			\centerline{\includegraphics[width=\columnwidth]{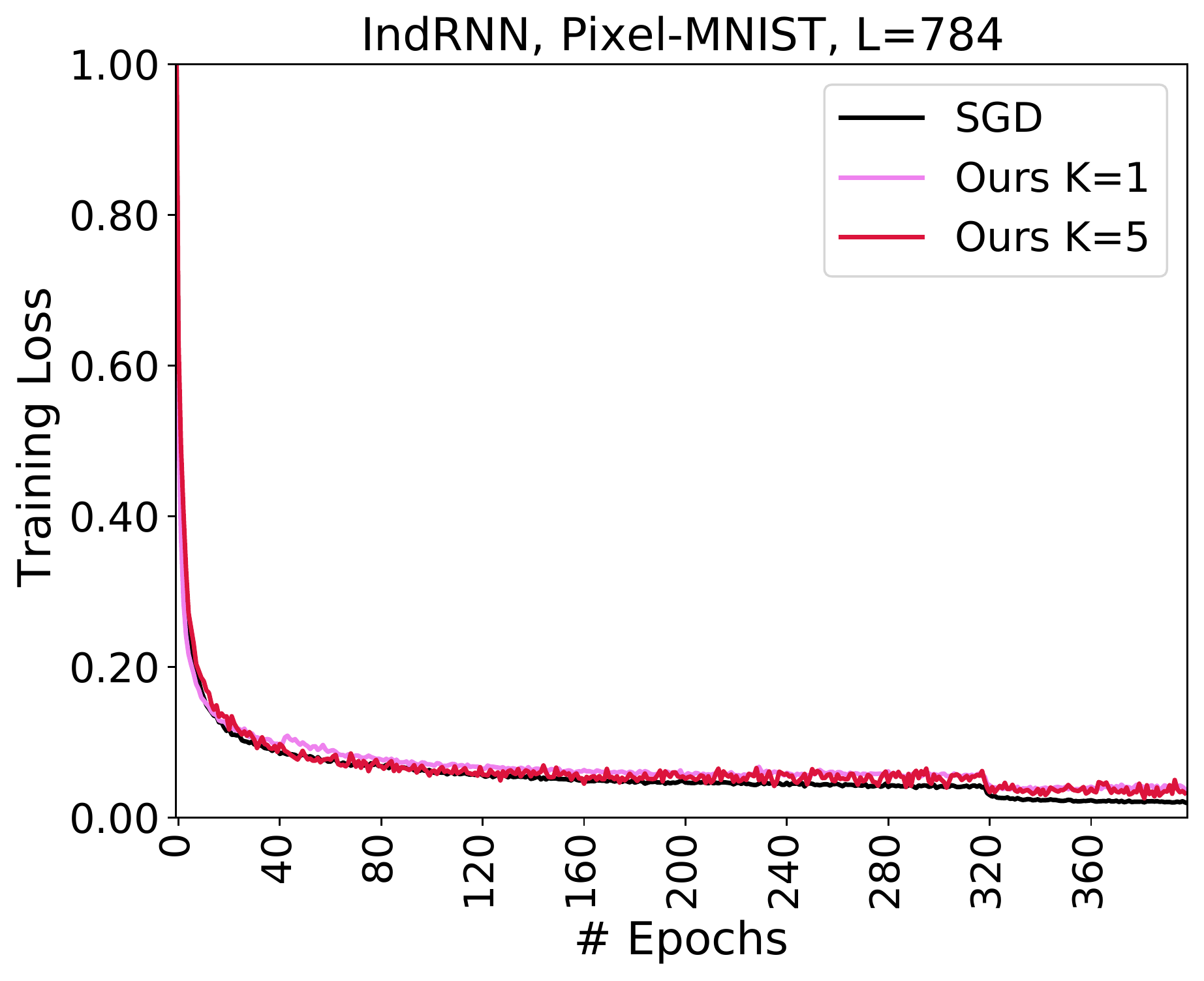}}
		\end{center}
	\end{minipage}
	\hfill
	\begin{minipage}[b]{0.45\columnwidth}
		\begin{center}
			\centerline{\includegraphics[width=\columnwidth]{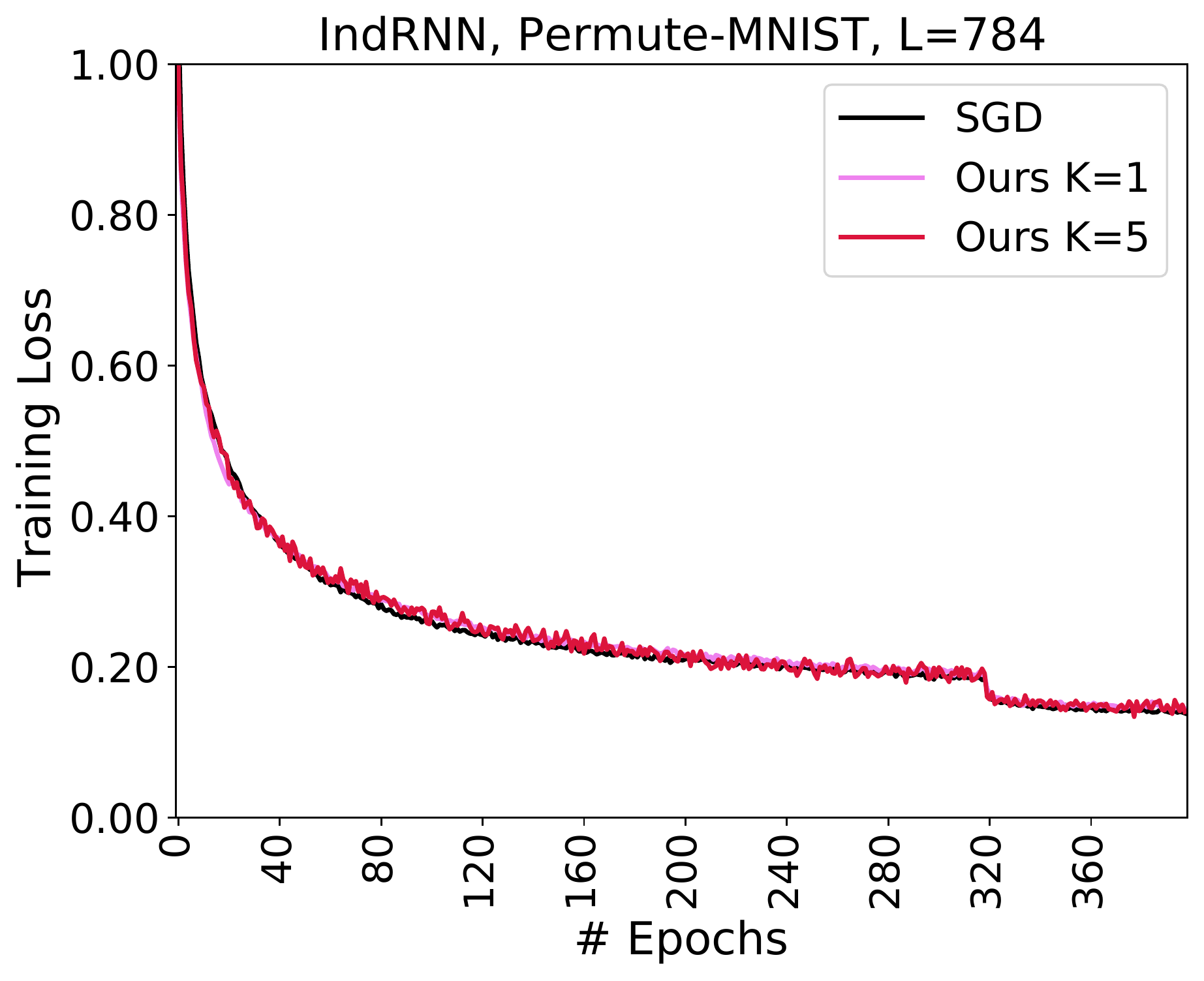}}
		\end{center}
	\end{minipage}
	\hfill
	\vspace{-5mm}
    \caption{\footnotesize Training loss on Pixel-MNIST and Permute-MNIST with indRNN}

\label{fig:indRNN}
\end{figure}

\begin{table}[hbt!]

  \def\arraystretch{0.8}%
  \begin{center}
    \caption{\footnotesize Test accuracy (\%) (training hours) of IndRNN}
    \label{tab:indRNN}
    \begin{tabular}{l c c c}
      \toprule 
      
         \multirow{2}{*}{\textbf{Dataset}}
 &  \multicolumn{3}{c}{\textbf{Acc. (Time)}}  \\
     \cmidrule(lr){2-4}
  & Baseline & Ours K=1 & Ours K=5 \\
      \midrule 
      Pixel-MNIST &
      98.88 (4.84) & 98.73 (3.46) & 98.82 (2.55)\\
      \midrule 
      Permute-MNIST
      & 93.00 (4.92) & 92.87 (3.68) & 92.59 (2.41)\\

      \bottomrule 
    \end{tabular}
  \end{center}

\end{table}

\bfsection{RNNs on HAR-2 \& Noisy HAR-2}
This experiment focuses on noise-free and noisy sequences. When the data samples are very noisy, RNNs usually exhibit unstable performance. To verify the robustness of our algorithm on noisy input, we compare the performance of the proposed optimizer with Adam and TBPTT on plain RNN with HAR-2 and Noisy-HAR-2 as input. The task is binary classification after observing a long sequence. HAR-2 has a 128 time sequence. Adding Gaussian noise with a mean of zero and a variance of two makes the task harder. We set forward and backward steps in TBPTT as 16. Thus the sequence is split into 16 segments. Same as Experiment 1, we also include the combination of TBPTT and our algorithm in the experiment. We also demonstrate the compatibility of our algorithm with LSTM and BN using the two datasets. We estimate the MSE of each experiment. 



\begin{figure}[t]
    \hfill
	\begin{minipage}[b]{0.45\columnwidth}
		\begin{center}
			\centerline{\includegraphics[width=\columnwidth]{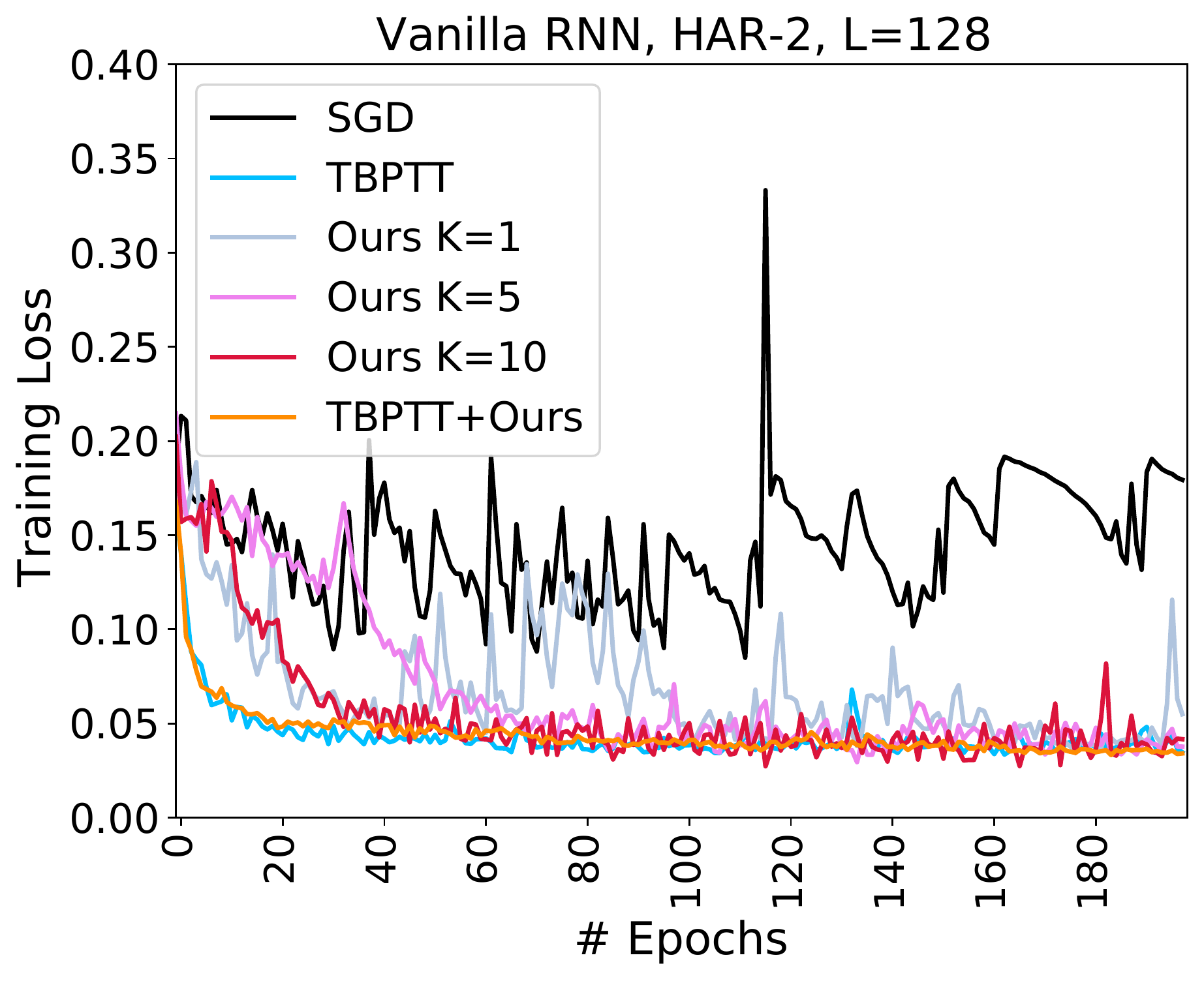}}
		\end{center}
	\end{minipage}
	\hfill
	\begin{minipage}[b]{0.45\columnwidth}
		\begin{center}
			\centerline{\includegraphics[width=\columnwidth]{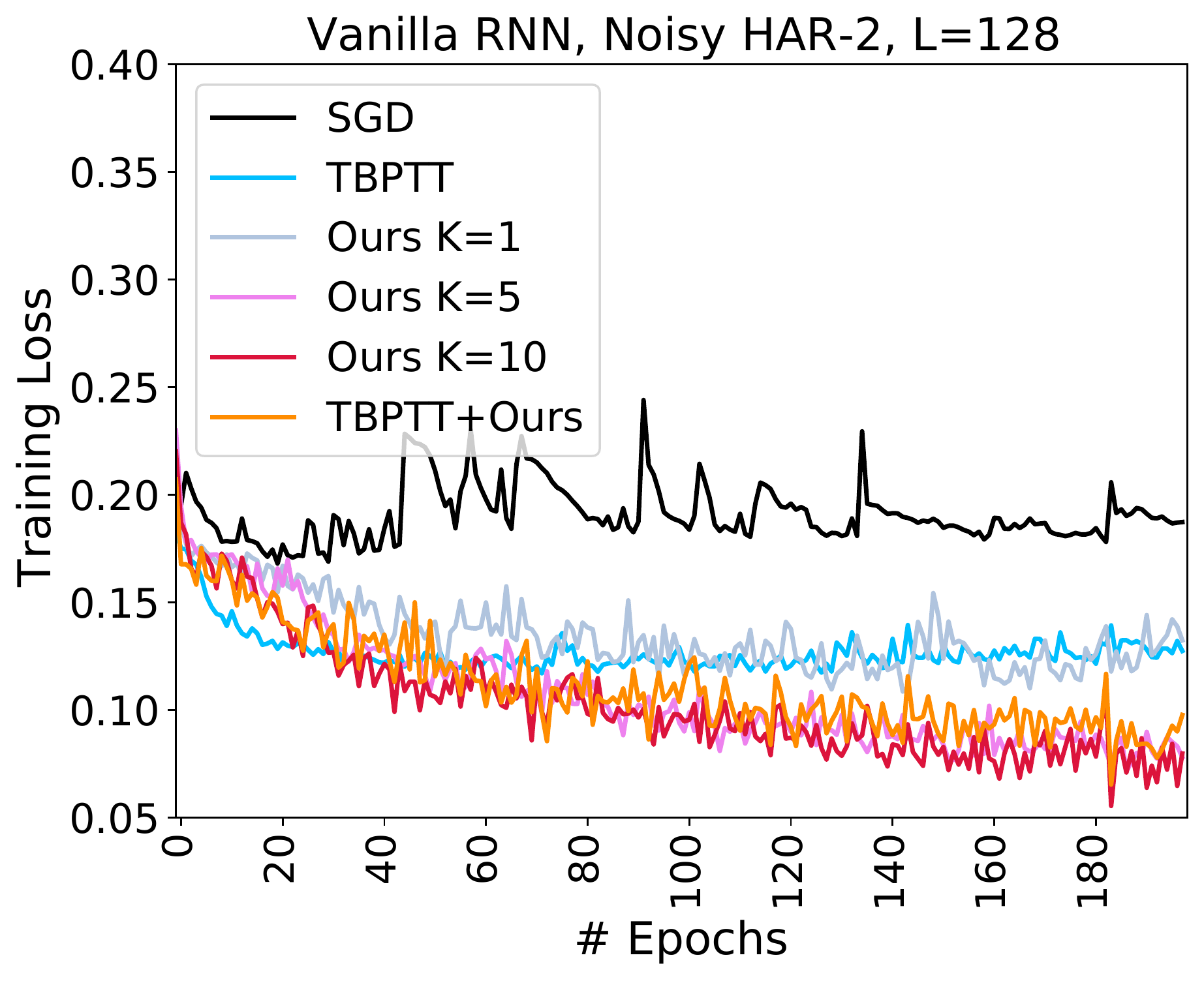}}
		\end{center}
	\end{minipage}
	\hfill
	\vspace{-5mm}
    \caption{\footnotesize Training loss of RNN on HAR-2 and Noisy-HAR-2}

\label{fig:har}
\end{figure}

\begin{table}[t]

  \caption{\footnotesize Test accuracy (\%) and training time (hr) of RNN}
  \label{tab:har}

  \centering

    \begin{tabular}{l c c c c }

      \toprule 
      \textbf{Method} & \textbf{HAR-2} & \textbf{Time} & \textbf {Noisy-HAR-2} & \textbf{Time}  \\
      \midrule 
      SGD & 87.66 & 0.17 & 74.38 & 0.17\\
      SGD+Clipping & 93.36 & 0.13 & 74.38 & 0.13 \\
      TBPTT & 93.62 & 0.38 & 86.20 & 0.56\\
      LSTM+Adam & 94.40 & 0.14 & 92.12 & 0.17\\
      Ours K=1 & 93.52 & 0.15 & 86.04 & 0.14 \\
      Ours K=5 & 94.11 & 0.14 & 89.36 & 0.14\\
      Ours K=10 & 93.65 & 0.37 & 89.52 & 0.35\\
      Ours+BN & 94.37 & 0.36 & 89.38 & 0.41\\ 
      TBPTT+Ours & 94.01 & 0.35 & 89.28 & 0.84\\
      LSTM+Ours & 94.95 & 0.19 & \textbf{92.41} & 0.42\\
      IndRNN & 95.73 & 0.46 & 91.20 & 0.45\\
      IndRNN+Ours & \textbf{96.55} & 0.13 & 92.15 & 0.17\\

      \bottomrule 
    \end{tabular}
\end{table} 

Fig. \ref{fig:har} and Table~\ref{tab:har} show that our algorithm is not only robust to long-term dependency tasks but also robust to noisy time sequences. No matter the input data is noise-free or noisy, the training losses of SGD oscillate up and down. The TBPTT reaches a desirable MSE in the noise-free case because the forward and backward steps are set up relatively short. TBPTT is capable of avoiding the vanishing gradient issue in this setting. The literature shows RNN with SGD usually reaches a test accuracy 91.31\% with the learning that takes at least 300 epochs on HAR-2 dataset \cite{kusupati2018nips}. With the same initial learning rate setting, our algorithm outperforms the baseline within less training epochs. It is shown in Table~\ref{tab:har} that gradient clipping does improve the model performance on noise-free data, however, it loses its advantage when applied to Noisy-HAR-2 dataset. When combining our algorithm with TBPTT, we reach the test accuracy 94.01\%. The baseline accuracy of SGD is 15.14\% less than our algorithm on the Noisy-HAR-2 task. We also verified that our optimizer works well with LSTM and BN as listed in Table~\ref{tab:har}. When our algorithm is combined with IndRNN, it overpasses the original IndRNN with less running time.

\section{Conclusions and future directions}\label{sec:conclusion}
In this paper, we propose a novel and simple RNN optimizer based on the Frank-Wolfe method. We provide a theoretical proof of the convergence of our algorithm. The empirical experiments on several benchmark datasets demonstrate that the proposed RNN optimizer is an effective solver for the training stability of RNNs. Our algorithm outperforms SGD in all experiments and boosts the TBPTT performances. It also reaches a comparable test accuracy with the baseline algorithm in deep RNNs while requiring fewer gradient updates and less training time. The algorithm shows robustness in the noisy data classification experiment with an improvement of 15.14\%. This work motivates the RNN training on a distributed system. In future work, we will investigate the application of our algorithm in a distributed setting which can reach significant speed-ups at no or nearly no loss of accuracy.






%



{\small
\bibliographystyle{IEEEtran}
\bibliography{root.bib}

}

\end{document}